\documentclass[a4paper,UKenglish]{article}

\usepackage{amsmath,amssymb}
\usepackage{eurosym}

\usepackage{amssymb}
\usepackage{epstopdf}

\usepackage{amsthm}
\usepackage{tikz}
\usetikzlibrary{fit,calc,positioning,decorations.pathreplacing,matrix, shapes}
\usepackage{multirow}
\usepackage{multicol}
\usepackage{graphicx}
\usepackage{latexsym}
\usepackage[lined,linesnumbered,boxed]{algorithm2e}
\usepackage{enumerate}
\usepackage{float}
\usepackage{array}
\usepackage{caption}
\usepackage{subcaption}
\usepackage{url}

\newcommand{\NP}{{\sf NP}}

\newcommand{\PSA}{{\sf PSA}}
\newcommand{\PCA}{{\sf PCA}}
\newcommand{\NSA}{{\sf NSA}}
\newcommand{\NCA}{{\sf NCA}}
\newcommand{\IncPV}{{\sf IncPV}}
\newcommand{\IncNV}{{\sf IncNV}}
\newcommand{\coNP}{{\sf coNP}}
\DeclareMathOperator{\cf}{cf}
\DeclareMathOperator{\ad}{ad}
\DeclareMathOperator{\gr}{gr}
\DeclareMathOperator{\stb}{st}
\DeclareMathOperator{\pr}{pr}
\DeclareMathOperator{\co}{co}
\DeclareMathOperator{\sk}{sk}
\DeclareMathOperator{\cred}{cr}
\DeclareMathOperator{\rej}{rej}
\DeclareMathOperator{\completions}{comp}

\makeatletter
\DeclareRobustCommand*\cal{\@fontswitch\relax\mathcal}
\makeatother

\newcommand{\iaf}{\ensuremath{{\cal I}}}
\newcommand{\riaf}{\ensuremath{{\cal R}}}

\newcommand{\af}{\ensuremath{{\cal F}}}

\newtheorem{theorem}{Theorem}
\newtheorem{proposition}[theorem]{Proposition}
\newtheorem{observation}[theorem]{Observation}
\newtheorem{definition}{Definition}
\newtheorem{example}{Example}

\begin{document}

\title{A Note on Rich Incomplete Argumentation Frameworks} 
\author{Jean-Guy Mailly\\LIPADE, Universit\'e de Paris,
  France\\jean-guy.mailly@u-paris.fr}
\date{}
\maketitle

\begin{abstract}
  Recently, qualitative uncertainty in abstract argumentation has
  received much attention. The first works on this topic introduced
  uncertainty about the presence of attacks, then about the presence
  of arguments, and finally combined both kinds of uncertainty. This
  results in the Incomplete Argumentation Framework (IAFs). But
  another kind of uncertainty was introduced in the context of Control
  Argumentation Frameworks (CAFs): it consists in a conflict relation
  with uncertain orientation, {\em i.e.} we are sure that there is an
  attack between two arguments, but the actual direction of the attack
  is unknown. Here, we formally define Rich IAFs, that combine the
  three different kinds of uncertainty that were previously introduced
  in IAFs and CAFs. We show that this new model, although strictly
  more expressive than IAFs, does not suffer from a blow up of
  computational complexity. Also, the existing computational approach
  based on SAT can be easily adapted to the new framework.
\end{abstract}

\section{Introduction}
Abstract argumentation \cite{Dung95} is an important topic in the
Knowledge Representation and Reasoning community. Intuitively, an
abstract argumentation framework (AF) is a directed graph where nodes
are arguments and edges are relations (usually attacks) between these
arguments. The outcome of such an AF is an evaluation of the
arguments' acceptance (through extensions \cite{Dung95,HOFASemantics},
labellings \cite{Caminada06} or rankings \cite{AmgoudB13a}). In such
an AF, the assumption of complete information is made: an argument
that appears in the graph is sure to actually exist, and similarly, an
edge (or the absence of an edge) in the graph means that the attack
between arguments certainly exists (or certainly does not).

The question of how to incorporate uncertainty in AFs has then
arisen. Two kinds of approaches have been proposed. If a quantitative
evaluation of the uncertainty is available, it seems natural to use it
in the definition of reasoning mechanisms. This corresponds ({\em
  e.g.}) to Probabilistic Argumentation Frameworks \cite{LiON11}. But
such a quantitative information about uncertainty may not be
available. The other approach is then the Incomplete Argumentation
Frameworks (IAFs)
\cite{Coste-MarquisDKLM07,BaumeisterNRS18,BaumeisterNR18}, where the
uncertainty is only qualitative. In an IAF, some arguments are
identified as uncertain, {\em i.e.} there is a doubt whether the
argument actually appears in the framework. Similarly, attacks may be
uncertain. However, another form of uncertainty in AFs has been
defined in the literature. Control Argumentation Frameworks
\cite{DimopoulosMM18} integrate uncertainty and argumentation dynamics
\cite{DoutreM18} in a single framework. Besides the two aforementioned
forms of uncertainty, a third one has been proposed: a symmetric
conflict relation is defined, such that there is an uncertainty about
the actual direction of the attack: either it appears in one
direction, or in the other one, or in both directions at the same
time. We investigate how this third kind of uncertainty
can be added to IAFs.

The report is organized as follows. Section~\ref{section:background}
describes the background notions on abstract argumentation and
Incomplete Argumentation Frameworks (IAFs). In
Section~\ref{section:riafs}, we introduce Rich Incomplete
Argumentation Frameworks (RIAFs), that generalize IAFs by adding a new
kind of uncertainty over the attacks. Section~\ref{section:conclusion}
concludes the report by mentioning several interesting research tracks
about (R)IAFs.

\section{Background}\label{section:background}
\subsection{Abstract Argumentation}\label{subsection:background-afs}
Abstract argumentation was introduced in \cite{Dung95}, where
arguments are abstract entities whose origin or internal structure are
ignored. The acceptance of arguments is purely defined from the
relations between them.

\begin{definition}[Abstract AF]
  An {\em abstract argumentation framework} (AF) is a directed graph
  $\af = \langle A, R \rangle$, where $A$ is a set of {\em arguments},
  and $R \subseteq A \times A$ is an {\em attack relation}. 
\end{definition}

We say that $a$ {\em attacks} $b$ when $(a,b) \in R$. If $(b,c) \in R$
also holds, then $a$ {\em defends} $c$ against $b$. Attack and defense
can be adapted to sets of arguments: $S \subseteq A$ attacks
(respectively defends) an argument $b \in A$ if $\exists a \in S$ that
attacks (respectively defends) $b$.

\begin{example}
  Let $\af = \langle A, R\rangle$ be the AF depicted at
  Figure~\ref{fig:example-af}, with $A = \{a,b,c,d,e\}$ and $R =
  \{(b,a), (c,a), (c,d), (d,b), (d,c), (e,a)\}$.
  \begin{figure}[ht]
    \centering
\begin{tikzpicture}[->,>=stealth,shorten >=1pt,auto,node distance=1.4cm,
                thick,main node/.style={circle,draw,font=\bfseries}]
\node[main node] (a) {$a$};
\node[main node] (b) [right of=a] {$b$};
\node[main node] (c) [below of=a] {$c$};
\node[main node] (d) [below of=b] {$d$};
\node[main node] (e) [left of=c] {$e$};

\path[->] (b) edge (a)
    (c) edge (a)
    (d) edge (b)
    (d) edge[bend right] (c)
    (c) edge[bend right] (d)
    (e) edge (a); 
\end{tikzpicture}
    \caption{The AF $\af$\label{fig:example-af}}
  \end{figure}
  Each arrow represents an attack. $d$ defends $a$ against both $b$
  and $c$, since these are attackers of $a$ that are, in turn, both
  attacked by $d$.
\end{example}

In \cite{Dung95}, Dung introduces different semantics to evaluate the
acceptability of arguments. They are based on two basic concepts:
\textit{conflict-freeness} and \textit{admissibility}.

\begin{definition}[Conflict-freeness and Admissibility]
Given $\af = \langle A, R\rangle$, a set of arguments $S \subseteq A$ is:
\begin{itemize}
\item conflict-free iff $\forall a, b \in S$, $(a,b) \not\in R$;
\item admissible iff it is conflict-free, and defends each $a \in S$ against each of its attackers.
\end{itemize}
\end{definition}

We use $\cf(\af)$ and $\ad(\af)$ for denoting the sets of conflit-free
and admissible sets of an argumentation framework $\af$.

The intuition behind these principles is that a set of arguments may
be accepted only if it is internally consistent (conflict-freeness)
and able to defend itself against potential threats (admissibility).
The semantics proposed by Dung are then defined as follows \cite{Dung95}.

\begin{definition}[Extension Semantics]
Given $\af = \langle A, R\rangle$, an admissible set $S \subseteq A$ is:
\begin{itemize}
	\item a complete extension iff it contains every argument that it defends;
	\item a preferred extension iff it is a $\subseteq$-maximal complete extension;
	\item the unique grounded extension iff it is the $\subseteq$-minimal complete extension;
	\item a stable extension iff it attacks every argument in $A \setminus S$.
\end{itemize}
\end{definition}

The sets of extensions of an AF $\af$, for these four semantics, are
denoted (respectively) $\co(\af)$, $\pr(\af)$, $\gr(\af)$ and
$\stb(\af)$.

Based on these semantics, we can define the status of any (set of)
argument(s), namely {\em skeptically accepted} (belonging to each
$\sigma$-extension), {\em credulously accepted} (belonging to some
$\sigma$-extension) and {\em rejected} (belonging to no
$\sigma$-extension). Given an AF $\af$ and a semantics $\sigma$, we
use (respectively) $\sk_\sigma(\af)$, $\cred_\sigma(\af)$ and
$\rej_\sigma(\af)$ to denote these sets of arguments.

\begin{example}
  We consider again $\af$ given at Figure~\ref{fig:example-af}. Its
  extensions for the different semantics, as well as the sets of
  accepted arguments, are given at
  Table~\ref{table:example-semantics}.
\begin{table}[ht]
  \centering
  \begin{tabular}{c|c|c|c}
    $\sigma$ & $\sigma(\af)$ & $\cred(\af)$ & $\sk(\af)$ \\ \hline
    $\co$ & $\{e\}, \{d,e\}, \{b,c,e\}$ & $\{b,c,d,e\}$ & $\{e\}$ \\ \hline
    $\pr$ & $\{d,e\}, \{b,c,e\}$ & $\{b,c,d,e\}$ & $\{e\}$ \\ \hline
    $\gr$ & $\{e\}$ & $\{e\}$ & $\{e\}$ \\ \hline
    $\stb$ & $\{d,e\}, \{b,c,e\}$ & $\{b,c,d,e\}$ & $\{e\}$
  \end{tabular}
  \caption{Extensions and Accepted Arguments of $\af$ for $\sigma \in \{\co, \pr, \gr, \stb\}$ \label{table:example-semantics}}
\end{table}
\end{example}

For more details about argumentation semantics, we refer the
interested reader to \cite{Dung95,HOFASemantics}.

Now, we introduce Incomplete Argumentation Frameworks
\cite{Coste-MarquisDKLM07,BaumeisterNRS18,BaumeisterNR18}, {\em i.e.}
AFs with qualitative uncertainty about the presence of some arguments
or attacks.

\begin{definition}[Incomplete AF]
An {\em Incomplete Argumentation Framework} (IAF) is a tuple $\iaf =
\langle A, A^?, R, R^?\rangle$, where $A$ and $A^?$ are disjoint sets
of arguments, and $R, R^? \subseteq (A \cup A^?) \times (A \cup A^?)$
are disjoint sets of attacks.
\end{definition}

Elements from $A$ and $R$ are certain arguments and attacks, {\em
  i.e.} the agent is sure that they appear in the framework. On the
opposite, $A^?$ and $R^?$ represent uncertain arguments and
attacks. For each of them, there is a doubt about their actual existence.

\begin{example}\label{example:example-iaf}
Let us consider $\iaf = \langle A, A^?, R, R^? \rangle$ given at
Figure~\ref{fig:example-iaf}. We use plain nodes and arrows to
represent certain arguments and attacks, {\em i.e.} $A =
\{a,b,c,d,e\}$ and $R = \{(b,a), (c,a), (d,b), (d,c)\}$. Uncertain
arguments are represented as dashed square nodes ({\em i.e.} $A^? =
\{f\}$) and uncertain attacks are represented as dotted arrows ({\em
  i.e.} $R^? = \{(e,a), (f,d)\}$).

  \begin{figure}[ht]
    \centering
\begin{tikzpicture}[->,>=stealth,shorten >=1pt,auto,node distance=1.4cm,
                thick,main node/.style={circle,draw,font=\bfseries},uncertain/.style={rectangle,draw,dashed,font=\bfseries}]
\node[main node] (a) {$a$};
\node[main node] (b) [right of=a] {$b$};
\node[main node] (c) [below of=a] {$c$};
\node[main node] (d) [below of=b] {$d$};
\node[main node] (e) [left of=c] {$e$};

\path[->] (b) edge (a)
    (c) edge (a)
    (d) edge (b)
    (d) edge (c);

\node[uncertain] (f) [right of=d] {$f$};

\path[->,dotted] (e) edge[bend left] (a);
\path[->,dotted] (f) edge (d);

\end{tikzpicture}
\caption{The IAF $\iaf$\label{fig:example-iaf}}
\end{figure}
\end{example}

The notion of completion in abstract argumentation was first defined
in \cite{Coste-MarquisDKLM07} for Partial AFs ({\em i.e.} IAFs with
$A^? = \emptyset$), and then adapted to IAFs. Intuitively, a
completion is a classical AF which describes a situation of the world
coherent with the uncertain information encoded in the IAF.

\begin{definition}[Completion of an IAF]  \label{def:completion-iaf}
  Given $\iaf = \langle A, A^?, R, R^? \rangle$, a completion of
  $\iaf$ is $\af$ = $\langle A', R' \rangle$, such that
  \begin{itemize}
  \item $A \subseteq A' \subseteq A \cup A^?$;
  \item $R_{\mid A'} \subseteq R' \subseteq R_{\mid A'} \cup
    R^?_{\mid A'}$;
  \end{itemize}
  where $R_{\mid A'} = R \cap (A' \times A')$ (and similarly for $R^?_{\mid A'}$).
\end{definition}

The set of completions of an IAF $\iaf$ is denoted $\completions(\iaf)$.

\begin{example}
We consider again the IAF from Figure~\ref{fig:example-iaf}. Its set
of completions is described at Figure~\ref{fig:example-completions}.
\begin{figure}[ht]
  \centering
  \begin{subfigure}[t]{0.25\textwidth}
    \centering
    \scalebox{0.65}{
\begin{tikzpicture}[->,>=stealth,shorten >=1pt,auto,node distance=1.4cm,
                thick,main node/.style={circle,draw,font=\bfseries},uncertain/.style={rectangle,draw,dashed,font=\bfseries}]
\node[main node] (a) {$a$};
\node[main node] (b) [right of=a] {$b$};
\node[main node] (c) [below of=a] {$c$};
\node[main node] (d) [below of=b] {$d$};
\node[main node] (e) [left of=c] {$e$};

\path[->] (b) edge (a)
    (c) edge (a)
    (d) edge (b)
    (d) edge (c);

\node (f) [right of=d] {};

\end{tikzpicture}
} 
\end{subfigure}
~
\vline
~
  \begin{subfigure}[t]{0.25\textwidth}
    \centering
    \scalebox{0.65}{
\begin{tikzpicture}[->,>=stealth,shorten >=1pt,auto,node distance=1.4cm,
                thick,main node/.style={circle,draw,font=\bfseries},uncertain/.style={rectangle,draw,dashed,font=\bfseries}]
\node[main node] (a) {$a$};
\node[main node] (b) [right of=a] {$b$};
\node[main node] (c) [below of=a] {$c$};
\node[main node] (d) [below of=b] {$d$};
\node[main node] (e) [left of=c] {$e$};

\path[->] (b) edge (a)
    (c) edge (a)
    (d) edge (b)
    (d) edge (c);

\node[main node] (f) [right of=d] {$f$};

\end{tikzpicture}
} 
\end{subfigure}
~
\vline
~
  \begin{subfigure}[t]{0.25\textwidth}
    \centering
    \scalebox{0.65}{
\begin{tikzpicture}[->,>=stealth,shorten >=1pt,auto,node distance=1.4cm,
                thick,main node/.style={circle,draw,font=\bfseries},uncertain/.style={rectangle,draw,dashed,font=\bfseries}]
\node[main node] (a) {$a$};
\node[main node] (b) [right of=a] {$b$};
\node[main node] (c) [below of=a] {$c$};
\node[main node] (d) [below of=b] {$d$};
\node[main node] (e) [left of=c] {$e$};

\path[->] (b) edge (a)
    (c) edge (a)
    (d) edge (b)
    (d) edge (c);

\node[main node] (f) [right of=d] {$f$};
\path[->] (f) edge (d);

\end{tikzpicture}
} 
\end{subfigure}

  \begin{subfigure}[t]{0.25\textwidth}
    \centering
    \scalebox{0.65}{
\begin{tikzpicture}[->,>=stealth,shorten >=1pt,auto,node distance=1.4cm,
                thick,main node/.style={circle,draw,font=\bfseries},uncertain/.style={rectangle,draw,dashed,font=\bfseries}]
\node[main node] (a) {$a$};
\node[main node] (b) [right of=a] {$b$};
\node[main node] (c) [below of=a] {$c$};
\node[main node] (d) [below of=b] {$d$};
\node[main node] (e) [left of=c] {$e$};

\path[->] (b) edge (a)
    (c) edge (a)
    (d) edge (b)
    (d) edge (c);

\node (f) [right of=d] {};
\path[->] (e) edge[bend left] (a);

\end{tikzpicture}
} 
\end{subfigure}
~
\vline
~
  \begin{subfigure}[t]{0.25\textwidth}
    \centering
    \scalebox{0.65}{
\begin{tikzpicture}[->,>=stealth,shorten >=1pt,auto,node distance=1.4cm,
                thick,main node/.style={circle,draw,font=\bfseries},uncertain/.style={rectangle,draw,dashed,font=\bfseries}]
\node[main node] (a) {$a$};
\node[main node] (b) [right of=a] {$b$};
\node[main node] (c) [below of=a] {$c$};
\node[main node] (d) [below of=b] {$d$};
\node[main node] (e) [left of=c] {$e$};

\path[->] (b) edge (a)
    (c) edge (a)
    (d) edge (b)
    (d) edge (c);

\node[main node] (f) [right of=d] {$f$};
\path[->] (e) edge[bend left] (a);

\end{tikzpicture}
} 
\end{subfigure}
~
\vline
~
  \begin{subfigure}[t]{0.25\textwidth}
    \centering
    \scalebox{0.65}{
\begin{tikzpicture}[->,>=stealth,shorten >=1pt,auto,node distance=1.4cm,
                thick,main node/.style={circle,draw,font=\bfseries},uncertain/.style={rectangle,draw,dashed,font=\bfseries}]
\node[main node] (a) {$a$};
\node[main node] (b) [right of=a] {$b$};
\node[main node] (c) [below of=a] {$c$};
\node[main node] (d) [below of=b] {$d$};
\node[main node] (e) [left of=c] {$e$};

\path[->] (b) edge (a)
    (c) edge (a)
    (d) edge (b)
    (d) edge (c);

\node[main node] (f) [right of=d] {$f$};
\path[->] (e) edge[bend left] (a);
\path[->] (f) edge (d);

\end{tikzpicture}
} 
\end{subfigure}
  \caption{The Completions of $\iaf$ \label{fig:example-completions}}
\end{figure}
\end{example}

The number of completions of an IAF $\iaf = \langle A, A^?, R,
R^?\rangle$ is bounded by $2^n$, with $n = |A^?| + |R^?|$. However,
this upper bound may not be reached, as it is the case in the
previous example. Indeed, the uncertain attack $(f,d)$ cannot appear
in completions where the uncertain argument $f$ does not appear.

\subsection{Reasoning with Incomplete AFs}\label{subsection:background-iafs}
To conclude this section, let us introduce the different reasoning
problems for IAFs that have been studied in the literature, as well as
their complexity.\footnote{We suppose that the reader is familiar with
basic concepts of computational complexity, like (non-)deterministic
polynomial algorithms, and the classes of the polynomial hierarchy: P,
\NP, \coNP, $\Sigma_k^P$, $\Pi_k^P$, where $k \in
\mathbb{N}$. Otherwise, we refer the interested reader to, {\em e.g.},
\cite{AroraBarak09}.}

They are the adaptation to IAFs of three classical
reasoning problems for AFs:
\begin{itemize}
  \item Verification: given an AF, a set of arguments, and a
    semantics, is the set an extension of the AF under the chosen
    semantics?
  \item Credulous acceptance: given an AF, an argument, and a
    semantics, is the argument a member of some extension under the
    chosen semantics?
  \item Skeptical acceptance: given an AF, an argument, and a
    semantics, is the argument a member of each extension under the
    chosen semantics?
\end{itemize}

Adapting these problems to IAFs requires to take into account the set
of completions. Indeed, an argument being accepted in one completion
is much less demanding than being accepted in all the
completions. This is why there are two variants of these problems for
IAFs: the possible and the necessary variant. The definition of the
possible variant quantifies existentially over the set of completions,
while the necessary variant quantifies universally.

Verification for IAFs was first studied in \cite{BaumeisterNRS18b}:
\begin{description}
\item[$\sigma$-\IncPV] Given $\iaf = \langle A, A^?, R, R^?\rangle$ an
  IAF and $S \subseteq A \cup A^?$, is there a completion $\af =
  \langle A', R' \rangle$ such that $S \cap A'$ is a
  $\sigma$-extension of $\af$?
\item[$\sigma$-\IncNV] Given $\iaf = \langle A, A^?, R, R^?\rangle$ an
  IAF and $S \subseteq A \cup A^?$, for each completion $\af =
  \langle A', R' \rangle$, is $S \cap A'$ a $\sigma$-extension of $\af$?
\end{description}

In \cite{FazzingaFF20}, a set of arguments for which the answer to
$\sigma$-\IncPV (respectively $\sigma$-\IncNV) is called a possible
(respectively necessary) $i$-extension. The authors identify some
issues with this definition (for instance, a set of arguments could be
identified as an $i$-extension even if it is not conflict-free). To
remedy this issue, they define so-called $i^*$-extensions, and the
corresponding verification problems:
\begin{description}
\item[$\sigma$-\IncPV$^*$] Given $\iaf = \langle A, A^?, R, R^?\rangle$ an
  IAF and $S \subseteq A \cup A^?$, is there a completion $\af =
  \langle A', R' \rangle$ such that $S$ is a
  $\sigma$-extension of $\af$?
\item[$\sigma$-\IncNV$^*$] Given $\iaf = \langle A, A^?, R, R^?\rangle$ an
  IAF and $S \subseteq A \cup A^?$, for each completion $\af =
  \langle A', R' \rangle$, is $S$ a $\sigma$-extension of $\af$?
\end{description}
We refer the interested reader to \cite{FazzingaFF20} for a detailled
discussion of the difference between $i$-extensions and
$i^*$-extensions. 

Finally, the (possible and necessary) variants of credulous and skeptical acceptance are studied
in \cite{BaumeisterNR18}:
\begin{description}
\item[$\sigma$-\PCA] Given $\iaf = \langle A, A^?, R, R^?\rangle$ an
  IAF and $a \in A$, is there a completion $\af =
  \langle A', R' \rangle$ such that $a$ is credulously accepted in
  $\af$ under $\sigma$?
\item[$\sigma$-\NCA] Given $\iaf = \langle A, A^?, R, R^?\rangle$ an
  IAF and $a \in A$, for each completion $\af =
  \langle A', R' \rangle$, is $a$ a credulously accepted in $\af$
  under $\sigma$?
\item[$\sigma$-\PSA] Given $\iaf = \langle A, A^?, R, R^?\rangle$ an
  IAF and $a \in A$, is there a completion $\af =
  \langle A', R' \rangle$ such that $a$ is skeptically accepted in
  $\af$ under $\sigma$?
\item[$\sigma$-\NSA] Given $\iaf = \langle A, A^?, R, R^?\rangle$ an
  IAF and $a \in A$, for each completion $\af =
  \langle A', R' \rangle$, is $a$ a skeptically accepted in $\af$
  under $\sigma$?
\end{description}

Now, let us give the complexity of these problems under several classical
semantics. $\mathcal{C}$-c means that the problem is complete for the
complexity class $\mathcal{C}$, under the given semantics.

{\small
\begin{table}[ht]
  \centering
  \begin{tabular}{c||c|c||c|c||c|c||c|c}
    $\sigma$ & \IncPV & \IncNV & \IncPV$^*$ & \IncNV$^*$ & \PCA & \NCA
    & \PSA & \NSA \\ \hline \hline
    $\ad$ & \NP-c & P & P & P & \NP-c & $\Pi_2^P$-c
    & trivial & trivial \\ \hline
    $\stb$ & \NP-c & P & P & P & \NP-c & $\Pi_2^P$-c
    & $\Sigma_2^P$-c & \coNP-c \\ \hline
    $\co$ & \NP-c & P & P & P & \NP-c & $\Pi_2^P$-c
    & \NP-c & \coNP-c \\ \hline
    $\gr$ & \NP-c & P & P & P & \NP-c & \coNP-c
    & \NP-c & \coNP-c \\ \hline
    $\pr$ & $\Sigma_2^P$-c & \coNP-c & $\Sigma_2^P$-c & \coNP-c & \NP-c & $\Pi_2^P$-c
    & $\Sigma_3^P$-c & $\Pi_2^P$-c
  \end{tabular}
  \caption{Complexity of IAFs for Various Problems
    under $\sigma \in \{\ad,\stb,\co,\gr,\pr\}$ \label{table:complexity-iafs}}
\end{table}
}

\section{Rich Incomplete Argumentation Frameworks}\label{section:riafs}
\subsection{Definition and Expressivity of RIAFs}

Now, we enrich the definition of IAFs.

\begin{definition}[Rich IAF]
  A {\em Rich Incomplete Argumentation Framework} \linebreak (RIAF) is a tuple
  $\riaf = \langle A, A^?, R, R^?, \leftrightarrow^?\rangle$, where
  $A$ and $A^?$ are disjoint sets of arguments, and
  $R, R^?, \leftrightarrow^? \subseteq (A \cup A^?) \times (A \cup
  A^?)$ are disjoint sets of attacks, such that $\leftrightarrow^?$ is
  symmetric.
\end{definition}

The new relation $\leftrightarrow^?$ is borrowed from Control
Argumentation Frameworks \cite{DimopoulosMM18}. It is a symmetric (uncertain) conflict
relation: if $(a,b) \in \leftrightarrow^?$, then we are sure that
there is a conflict between $a$ and $b$, but not of the direction of
the attack. This new relation impacts the definition of completions.

\begin{definition}[Completion of a RIAF]
  Given $\riaf = \langle A, A^?, R, R^?, \leftrightarrow^? \rangle$, a
  completion of $\riaf$ is $\af$ = $\langle A', R' \rangle$, such that
  \begin{itemize}
  \item $A \subseteq A' \subseteq A \cup A^?$;
  \item $R_{\mid A'} \subseteq R' \subseteq R_{\mid A'} \cup R^?_{\mid A'} \cup \leftrightarrow^?_{\mid A'}$;
  \item if $(a,b) \in \leftrightarrow^?_{\mid A'}$, then $(a,b) \in R'$ or
    $(b,a) \in R'$ (or both);
  \end{itemize}
  where $R_{\mid A'} = R \cap (A' \times A')$ (and similarly for $R^?_{\mid A'}$ and $\leftrightarrow^?_{\mid A'}$).
\end{definition}

Again, we use $\completions(\riaf)$ to denote the set of completions
of a RIAF $\riaf$.

\begin{example}
We present a slight modification of the IAF from
Example~\ref{example:example-iaf}, where the (certain) attack $(b,a)$
is replaced by a symmetric uncertain conflict between $a$ and $b$. The
resulting RIAF is given at Figure~\ref{fig:example-riaf}.
  
  \begin{figure}[ht]
    \centering
\begin{tikzpicture}[->,>=stealth,shorten >=1pt,auto,node distance=1.4cm,
                thick,main node/.style={circle,draw,font=\bfseries},uncertain/.style={rectangle,draw,dashed,font=\bfseries}]
\node[main node] (a) {$a$};
\node[main node] (b) [right of=a] {$b$};
\node[main node] (c) [below of=a] {$c$};
\node[main node] (d) [below of=b] {$d$};
\node[main node] (e) [left of=c] {$e$};

\path[->] (c) edge (a)
    (d) edge (b)
    (d) edge (c);

\node[uncertain] (f) [right of=d] {$f$};

\path[->,dotted] (e) edge[bend left] (a);
\path[->,dotted] (f) edge (d);
\path[<->,dashed] (a) edge (b);

\end{tikzpicture}
\caption{The RIAF $\riaf$\label{fig:example-riaf}}
\end{figure}

We do not give here the full set of completions of $\riaf$. Let us
focus on one option for each of $(e,a)$, $(f,d)$ and $f$, and we only
illustrate the three options for $(a,b)$. These three completions are
given at Figure~\ref{fig:example-riaf-completions}.

\begin{figure}[ht]
  \centering
  \begin{subfigure}[t]{0.25\textwidth}
    \centering
    \scalebox{0.65}{
\begin{tikzpicture}[->,>=stealth,shorten >=1pt,auto,node distance=1.4cm,
                thick,main node/.style={circle,draw,font=\bfseries},uncertain/.style={rectangle,draw,dashed,font=\bfseries}]
\node[main node] (a) {$a$};
\node[main node] (b) [right of=a] {$b$};
\node[main node] (c) [below of=a] {$c$};
\node[main node] (d) [below of=b] {$d$};
\node[main node] (e) [left of=c] {$e$};

\path[->] (b) edge (a)
    (c) edge (a)
    (d) edge (b)
    (d) edge (c);

\node[main node] (f) [right of=d] {$f$};
\path[->] (e) edge[bend left] (a);
\path[->] (f) edge (d);

\end{tikzpicture}
} 
\end{subfigure}
~
\vline
~
  \begin{subfigure}[t]{0.25\textwidth}
    \centering
    \scalebox{0.65}{
\begin{tikzpicture}[->,>=stealth,shorten >=1pt,auto,node distance=1.4cm,
                thick,main node/.style={circle,draw,font=\bfseries},uncertain/.style={rectangle,draw,dashed,font=\bfseries}]
\node[main node] (a) {$a$};
\node[main node] (b) [right of=a] {$b$};
\node[main node] (c) [below of=a] {$c$};
\node[main node] (d) [below of=b] {$d$};
\node[main node] (e) [left of=c] {$e$};

\path[->] (a) edge (b)
    (c) edge (a)
    (d) edge (b)
    (d) edge (c);

\node[main node] (f) [right of=d] {$f$};
\path[->] (e) edge[bend left] (a);
\path[->] (f) edge (d);

\end{tikzpicture}
} 
\end{subfigure}
~
\vline
~
  \begin{subfigure}[t]{0.25\textwidth}
    \centering
    \scalebox{0.65}{
\begin{tikzpicture}[->,>=stealth,shorten >=1pt,auto,node distance=1.4cm,
                thick,main node/.style={circle,draw,font=\bfseries},uncertain/.style={rectangle,draw,dashed,font=\bfseries}]
\node[main node] (a) {$a$};
\node[main node] (b) [right of=a] {$b$};
\node[main node] (c) [below of=a] {$c$};
\node[main node] (d) [below of=b] {$d$};
\node[main node] (e) [left of=c] {$e$};

\path[->] (b) edge[bend left] (a)
    (a) edge[bend left] (b)
    (c) edge (a)
    (d) edge (b)
    (d) edge (c);

\node[main node] (f) [right of=d] {$f$};
\path[->] (e) edge[bend left] (a);
\path[->] (f) edge (d);

\end{tikzpicture}
} 
\end{subfigure}
\caption{Three Completions of $\riaf$ \label{fig:example-riaf-completions}}
\end{figure}
Similarly, for each other configuration of $(e,a)$, $(f,d)$ and $f$
({\em i.e.} each completion at Figure~\ref{fig:example-completions}),
there are three options for the conflict between $a$ and $b$, leading
to three different completions.
\end{example}

\begin{observation}
Let us notice that $\leftrightarrow^?$ is not defined as a symmetric
relation in \cite{DimopoulosMM18}. However, its meaning imposes the
relation to be symmetric. Indeed, $(a,b) \in \leftrightarrow^?$ means
that there is a conflict between $a$ and $b$ whose direction is
uncertain. Formally, it means that ({\em ceteris paribus}) there are
three completions with (respectively) $(a,b) \in R'$, or $(b,a) \in
R'$, or both. This is obviously equivalent to ``there is a conflict
between $b$ and $a$ whose direction is uncertain'', {\em i.e.} $(b,a)
\in \leftrightarrow^?$. A non-symmetric relation can be used as a more
compact representation of its symmetric counterpart.
\end{observation}

Now, we prove that RIAFs are strictly more expressive than IAFs. Said
otherwise, it means that the new relation $\leftrightarrow^?$ cannot
be equivalently represented with a combination of fixed and uncertain
attacks.

\begin{proposition}[Relative Expressivity of IAFs and RIAFs]\label{prop:relative-expressivity}\ \\
  \vspace*{-0.3cm}
\begin{itemize}
\item For any IAF $\iaf$, there exists a RIAF $\riaf$ such that
  $\completions(\iaf) = \completions(\riaf)$.
\item There exists a RIAF $\riaf$ such that there is no IAF $\iaf$
  with $\completions(\iaf) = \completions(\riaf)$.
\end{itemize}
\end{proposition}

\begin{proof}
  The first item is straightforward: any IAF is a RIAF with
  $\leftrightarrow^? = \emptyset$. For the second item, consider
  $\riaf = \langle \{a,b\}, \emptyset, \emptyset, \emptyset, \{(a,b),
  (b,a)\} \rangle$. This RIAF and its three completions are given at
  Figure~\ref{fig:proof-expressivity-riaf}.
\begin{figure}[ht]
  \centering
  \begin{subfigure}[t]{0.25\textwidth}
    \centering
    \newcommand{\shiftpoints}{5pt}
    \begin{tikzpicture}[->,>=stealth,shorten >=1pt,auto,node
      distance=1.4cm, thick,main
      node/.style={circle,draw,font=\bfseries},uncertain/.style={rectangle,draw,dashed,font=\bfseries}]
      \node[main node] (a) {$a$};
      \node[main node] (b) [right of=a] {$b$};
      
      \path[<->,dashed] (a) edge (b);
      
    \end{tikzpicture}
  \end{subfigure}
  ~
  \vline
  ~
    \begin{subfigure}[t]{0.25\textwidth}
    \centering
    \newcommand{\shiftpoints}{5pt}
    \begin{tikzpicture}[->,>=stealth,shorten >=1pt,auto,node
      distance=1.4cm, thick,main
      node/.style={circle,draw,font=\bfseries},uncertain/.style={rectangle,draw,dashed,font=\bfseries}]
      \node[main node] (a) {$a$};
      \node[main node] (b) [right of=a] {$b$};
      
      \path[->] (a) edge (b);
      
    \end{tikzpicture}
  \end{subfigure}

    \begin{subfigure}[t]{0.25\textwidth}
    \centering
    \newcommand{\shiftpoints}{5pt}
    \begin{tikzpicture}[->,>=stealth,shorten >=1pt,auto,node
      distance=1.4cm, thick,main
      node/.style={circle,draw,font=\bfseries},uncertain/.style={rectangle,draw,dashed,font=\bfseries}]
      \node[main node] (a) {$a$};
      \node[main node] (b) [right of=a] {$b$};
      
      \path[->] (b) edge (a);
      
    \end{tikzpicture}
  \end{subfigure}
    ~
  \vline
  ~
    \begin{subfigure}[t]{0.25\textwidth}
    \centering
    \newcommand{\shiftpoints}{5pt}
    \begin{tikzpicture}[->,>=stealth,shorten >=1pt,auto,node
      distance=1.4cm, thick,main
      node/.style={circle,draw,font=\bfseries},uncertain/.style={rectangle,draw,dashed,font=\bfseries}]
      \node[main node] (a) {$a$};
      \node[main node] (b) [right of=a] {$b$};
      
      \path[->] (a) edge[bend left] (b)
      (b) edge[bend left] (a);
      
    \end{tikzpicture}
  \end{subfigure}
  \caption{A RIAF and its Completions \label{fig:proof-expressivity-riaf}}
\end{figure}
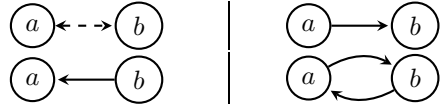

Now, let us prove that there is no IAF with the same set of
completions. Reasoning towards a contradiction, suppose that such a
IAF $\iaf = \langle A, A^?, R, R^? \rangle$ exists. Since all the
completions have the same set of arguments $\{a,b\}$, there cannot be
uncertain argument, {\em i.e.} $A^? = \emptyset$.

Let us now consider the different options for $R$ and $R^?$. If
$(a,b) \in R$ (respectively $(b,a) \in R$), then there is an attack
from $a$ to $b$ (respectively from $b$ to $a$) in every
completion. This is not the case. Similarly, there cannot be any self
attack in $R$ (since there is no such attack in any completion). Thus
$R = \emptyset$.

In the case where only $(a,b)$ (respectively $(b,a)$)
belongs to $R^?$, then the completions with $(b,a)$ (respectively
$(a,b)$) do not belong to $\completions(\iaf)$. On the contrary, if
both $(a,b)$ and $(b,a)$ belong to $R^?$, then a fourth completion
where there is no attack between $a$ and $b$ belongs to
$\completions(\iaf)$. Of course, self-attacks in $R^?$ are not
possible, since they would yield addition completions (with the same
self-attack appearing in them).

So we can conclude that $\iaf$ does not exist.
\end{proof}

\subsection{Computational Issues}
\subsubsection{Complexity}
While we have shown that RIAFs are strictly more expressive than IAFs,
now we prove that this expressivity is not at the price of a
complexity blow up. Let us recall that the complexity results for IAFs
\cite{BaumeisterNRS18b,BaumeisterNR18,FazzingaFF20} are summarized at
Table~\ref{table:complexity-iafs}. The fact that any IAF is a RIAF
with $\leftrightarrow^? = \emptyset$ is enough to prove that reasoning
with RIAFs is at least as hard as reasoning with IAFs. But also, we
notice that the upper bounds of the complexity coincides with the
upper bound for IAFs. Roughly speaking, this can be explained by the
fact that non-deterministically guessing a completion is not different
for RIAFs than for IAFs.

\paragraph{Possible Verification} Let us start with \IncPV{} and
\IncPV$^*$. The problem for RIAFs can be solved by the
non-deterministic guess of a completion, and then checking whether the
queried set of arguments $S$ is a $\sigma$-extension of the given
completion (for \IncPV$^*$), or $S \cap A'$ where $A'$ is the set of
arguments that appear in the completion (for \IncPV). Since
verification of an extension in AFs is polynomial for
$\sigma \in \{\ad,\stb,\co,\gr\}$ and \coNP-complete for
$\sigma = \pr$, the results for \IncPV{} from
Table~\ref{table:complexity-iafs} are valid for RIAFs, as well as the
result for \IncPV$^*$ under the preferred semantics. For \IncPV$^*$
under the other semantics, the reasoning from \cite{FazzingaFF20}
applies. First, if there are $(a,b),(b,a) \in \leftrightarrow^?$ such
that $a,b \in S$, then $S$ is not conflict-free in $\riaf$, thus it is
not an extension (for $\sigma = \ad,\stb,\gr,\co$). Otherwise, for
$(a,b),(b,a) \in \leftrightarrow^?$ with $a \in S$ and $b \not\in S$,
include only $(a,b)$ in the completion built by the algorithm.

  \paragraph{Necessary Verification} Then, let us focus on \IncNV$^*$,
  and $\sigma = \pr$. A negative instance for this problem can be
  identified by non-deterministically guessing a completion and a
  superset $S'$ of the queried set of arguments $S$. Then,
  polynomially checking whether $S'$ is admissible allows to conclude
  that $S$ is not necessary a preferred $i^*$-extension. Thus
  \IncNV$^* \in \coNP$. For \IncNV, the reasoning is the same, except
  that $S'$ must be a superset of $S \cap A'$ instead of a superset of
  $S$.
  
  For the other semantics, the reasoning from
  \cite{BaumeisterNRS18b,FazzingaFF20} holds for RIAFs: \IncNV{} and
  \IncNV$^*$ can be solved polynomially by reducing the problem to
  reasoning with Argument-Incomplete AFs, or by constructing the
  adequate completion. If there are
  $(a,b),(b,a) \in \leftrightarrow^?$ such that $a,b \in S$, then $S$
  is not conflict-free in $\riaf$, thus it is not an extension (for
  $\sigma = \ad,\stb,\gr,\co$). Otherwise, for
  $(a,b),(b,a) \in \leftrightarrow^?$ with $a \in S$ and
  $b \not\in S$, include only $(a,b)$ in the Argument-Incomplete AF
  or completion built in the proof.

  \paragraph{Credulous Acceptance }Now, we look at acceptance
  problems. For \PCA, the problem is solved by non-deterministically
  guessing a completion $C$ and a set of arguments $S$ that contains
  the queried argument $a$. Then, it can be checked polynomially
  whether $S$ is a $\sigma$-extension, for
  $\sigma \in \{\ad,\stb,\co,\gr\}$. Moreover, if $S$ is an admissible
  extension, then $a$ also belongs to some preferred extension (since
  each admissible set is included in some preferred extension).

For \NCA, let us non-deterministically guess a completion and check
whether the queried argument is credulously accepted in it. This check
is doable in polynomial time for $\sigma = \gr$, and with a \NP{} oracle
for the other semantics under consideration, hence the result.

\paragraph{Skeptical Acceptance} For \PSA{} and \NSA, the admissible
semantics remains a trivial case: since the empty set is admissible
for any AF, there is no skeptically accepted argument under $\ad$. For
\PSA, we can non-deterministically guess a completion, and check
whether $a$ is skeptically accepted in it. This check is polynomial
for $\sigma \in \{\gr,\co\}$, in \coNP{} for $\sigma = \stb$, and in
$\Pi_2^P$ for $\sigma = \pr$, so we obtain the result. For \NSA, the
reasoning described in \cite{BaumeisterNR18} still holds: checking the
necessary skeptical acceptance of an argument can be represented by a
universal quantification over the set of completions and the set of
sets of arguments ($S$) that do not contain the queried argument
$a$. The universal quantifiers are followed by a (deterministic)
polynomial check for $\sigma \in \{\stb,\co,\gr\}$. For
$\sigma = \pr$, the universal quantifiers are followed by an
existential quantifier over the supersets of $S$, and finally a
(deterministic) polynomial check. Hence the results.

So we can conclude that the complexity of reasoning with RIAFs, for
the various problems introduced in Section~\ref{subsection:background-iafs}
and for $\sigma \in \{\ad,\stb,\co,\gr,\pr\}$, is the same as in the
case of IAFs.

\begin{proposition}
The complexity results for IAFs given at
Table~\ref{table:complexity-iafs} also hold for RIAFs.
\end{proposition}

\subsubsection{Algorithm}\label{section:algorithm}
Now we show how to adapt the SAT-based algorithms for reasoning with
IAFs \cite{NiskanenNJR20} to RIAFs.

Roughly speaking, the encoding is made of one part that represents the
structure of the IAF, {\em i.e.} the existence of (uncertain or not)
arguments and attacks; and one part that maps this structure with the
arguments acceptance evaluation (with respect to a chosen
semantics). A slight modification to take into account
$\leftrightarrow^?$ is enough to reason with RIAFs instead of
IAFs. Following the definition of encodings in \cite{NiskanenNJR20},
we define, for a RIAF
$\riaf = \langle A, A^?, R, R^?, \leftrightarrow^?\rangle$ the Boolean
variables:
\begin{itemize}
\item $y_a$ is true if and only if $a \in A'$;
\item $r_{a,b}$ is true if and only if $(a,b) \in R'$;
\item $x_a$ is true if and only if $a \in S$, for some $S \in \sigma(\af)$;
\end{itemize}
where $\af = \langle A', R'\rangle$ is a completion of $\riaf$
corresponding to the $y_a$ and $r_{a,b}$ variables. While $y_a$ and
$r_{a,b}$ are necessarily true for $a \in A$ and $(a,b) \in R$, they
may be true or false for $a \in A^?$ and $(a,b) \in R^?$. We also need
to specify that either $r_{a,b}$ or $r_{b,a}$ is true if $(a,b),(b,a)
\in \leftrightarrow^?$. Finally, an argument $a$ must be in the
completion ({\em i.e.} $y_a$ is true) in order to appear in an
extension ({\em i.e.} $x_a$ is true) or in some attacks ({\em i.e.}
$r_{a,b}$ or $r_{b,a}$ is true for some $b$). This corresponds to the
following formula, for $\riaf = \langle A, A^?, R, R^?,
\leftrightarrow^? \rangle$:
\[
  \begin{array}{rcl}
    \phi_?(\riaf) & = & (\bigwedge_{a \in A} y_a ) \wedge (\bigwedge_{(a,b)
    \in R} r_{a,b}) \wedge (\bigwedge_{(a,b) \in \leftrightarrow^?}
                        r_{a,b} \vee r_{b,a}) \\
    & & \wedge (\bigwedge_{a \in A^?}(\neg y_a \rightarrow (\neg x_a
        \wedge \bigwedge_{(a,b) \in R^?} \neg r_{a,b} \wedge
        \bigwedge_{(b,a) \in R^?} \neg r_{b,a})))
  \end{array}
\]

Then, a formula encoding the semantics is given, for instance:
\[
  \phi_{\cf}(\riaf) = \bigwedge_{(a,b) \in R \cup R^? \cup
    \leftrightarrow^?} (y_a \wedge y_b \wedge r_{a,b}) \rightarrow
  (\neg x_a \vee \neg x_b)
\]
encodes conflict-freeness: if both arguments appear in the completion,
as well as the attack between them, then they cannot be both
accepted. Similarly, $\phi_{\ad}$ and $\phi_{\stb}$ are provided in
\cite{NiskanenNJR20}, we show their adaptation to RIAFs:
\begin{itemize}
\item $\phi_{\ad}(\riaf) = \phi_{\cf}(\riaf) \wedge \bigwedge_{a \in A
    \cup A^?} \bigwedge_{(b,a) \in R \cup R^? \cup
    \leftrightarrow^?} ((x_a \wedge y_a \wedge y_b \wedge r_{b,a})
  \rightarrow z_b)$,
\item $\phi_{\stb}(\riaf) = \phi_{\cf}(\riaf) \wedge \bigwedge_{a \in
    A\cup A^?} (\neg x_a \wedge y_a \rightarrow z_a)$,
\end{itemize}
 where $z_a$ is a newly introduced variable for each argument $a \in A
 \cup A^?$, meaning that $a$ is defeated. This is formally encoded by
 $z_a \rightarrow \bigvee_{(b,a) \in R \cup R^? \cup \leftrightarrow^?} (x_b
 \wedge y_b \wedge r_{b,a})$.

$\phi_{ad}$ and $\phi_{\stb}$ can be used directly for solving problems
at the first level of the polynomial hierarchy, or as \NP{} abstraction
in CounterExample Guided Abstract Refinement algorithms. We refer the
interested reader to \cite{NiskanenNJR20} for more details.

\section{Conclusion}  \label{section:conclusion}
This report introduces Rich Incomplete Argumentation Frameworks, that
generalize IAFs by adding a new kind of uncertainty. We have shown
that this model is strictly more expressive than IAFs, but not at the
price of an increase of complexity. Moreover, a slight modification of
existing logical encodings allows to use the algorithms described in
the literature.

While complexity of reasoning with (R)IAFs has been well studied,
there are still many open questions regarding this formalism. For
instance, as far as we know, the only algorithms proposed (and
implemented) for IAFs concern the acceptance problems, as mentioned in
the Section~\ref{section:algorithm}. Other problems mentioned in
Section~\ref{subsection:background-iafs} have not been tackled
yet. Also, this study only considers the initial semantics defined by
Dung, but other semantics have not received interest in the context of
(R)IAFs, {\em e.g.} semi-stable \cite{CaminadaCD12}, stage
\cite{Verheij96} or ideal \cite{DungMT07} semantics.\\
Several works about the revision \cite{CMKMM2014a,CMKMM2014b}, the
update \cite{DoutreHP14} or the merging
\cite{Coste-MarquisDKLM07,DHKMRW2016} of AFs have faced the difficulty
to represent the uncertainty of the result as a single AF, and chose
to return a set of AFs as output. Let us recall that Partial AFs (that
are a special case of RIAFs with $A^? = \emptyset$ and
$\leftrightarrow^?$) were defined as a part of the process for
computing the merging of AFs \cite{Coste-MarquisDKLM07}, but did not
appear in the result of the operation. RIAFs might provide an
interesting solution in order to have a more compact output for these operations. Related
to these questions, the issue of realizability of extension sets
\cite{DunneDLW15} in the context of (R)IAFs is also interesting.



\bibliography{rich_iafs}

\end{document}